\documentclass[12pt,twoside]{article}
\usepackage{inputenc}
\usepackage[english]{babel} 
\usepackage{amsfonts}
\usepackage{amsmath}
\usepackage{amssymb}
\usepackage{amsthm}
\usepackage{graphicx}          
\usepackage{float}
\usepackage{caption}
\usepackage{textcomp}

\usepackage[]{algorithm2e}

\begin{document}

\newtheorem{lemma}{Lemma}
\newtheorem{theorem}{Theorem}

\author{ E.E. Ivanko\\  \footnotesize a) Institute of Mathematics and Mechanics, Ural branch, RAS\\  \footnotesize b) Ural Federal University\\ \footnotesize evgeny.ivanko@gmail.com}
\title{The destiny of constant structure discrete time closed semantic systems}
\date{}
\maketitle{}

\begin{abstract}
\noindent {\footnotesize \qquad Constant structure closed semantic systems are the systems each element of which receives its definition through the correspondent unchangeable set of other elements of the system. Discrete time means here that the definitions of the elements change iteratively and simultaneously based on the ``neighbor portraits''  from the previous iteration. I prove that the iterative redefinition process in such class of systems will quickly degenerate into a series of pairwise isomorphic states and discuss some directions of further research.

}
\end{abstract}

\markboth{E.E. Ivanko}{The destiny of closed semantic systems}


\section*{Introduction}
\hspace{0.7 cm}A closed semantic system (CSS) may be thought as a system each element of which is defined through other elements of this system. One of the most natural and important CSSs is language.
Every child at the age of 3--5 years is full of questions: ``Why?'' , ``What for?'' , ``How?''\  \cite{bloom}. At that time, the child's world view is growing and getting as closed as possible: every word claims to be explained in terms of other words. 

One of the most simple CSSs is a discrete time system with a permanent structure: 1) the number of the elements that are involved in the definition of each element of such system does not change with time; 2) all the elements are redefined simultaneously, basing on the states of the neighbors taken at the previous simultaneous iteration. In this paper, I study the behavior of such trivial systems, starting from the ``zero''\  state, where all the system's elements are equal. 
In contrast with the above mentioned growing CSS of small kids, a CSS with a constant structure 
resembles an adult's world view, where the addition of new notions and connections between them ceases as the personal world view approaches the best known contemporary world view of the humanity. The simultaneous iterative changes in discrete time is a reasonable assumption as far as we consider an artificial CSS designed for a deterministic Turing machine equivalent \cite{papa}. However, the discreteness can hardly correspond to real-world examples, so it would be challenging to get rid of it in the future.

\begin{figure}[hp]
\centering
\captionsetup{justification=centering}
\includegraphics[width=0.7\textwidth]{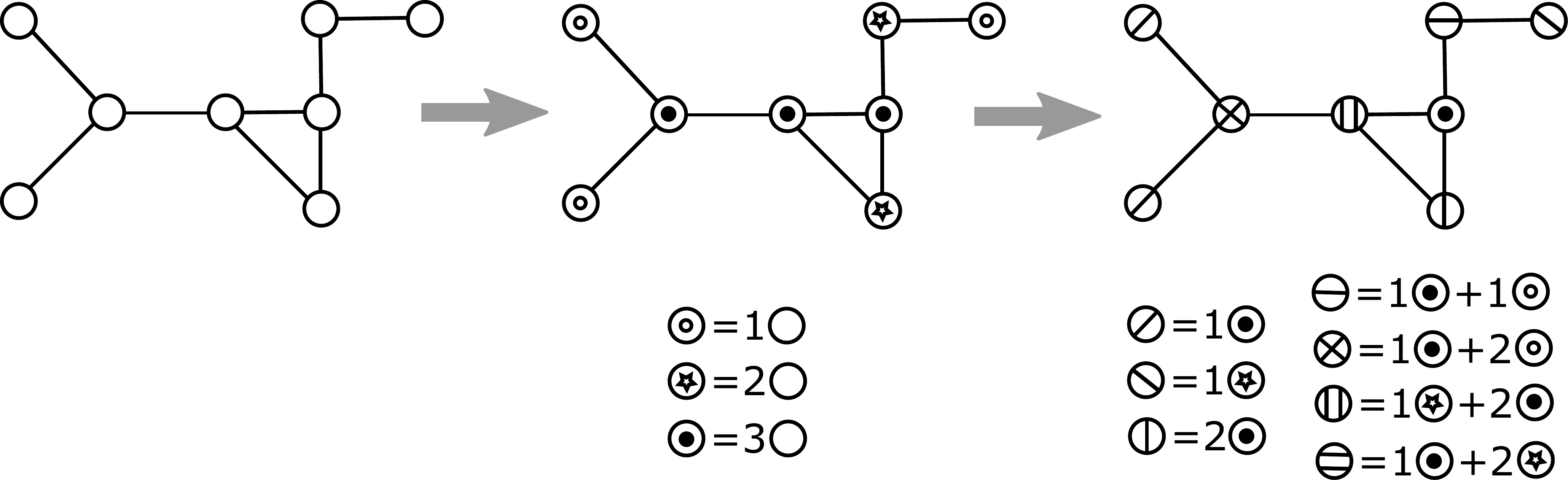}
\caption{Two iterations of the coloring process}
\end{figure}
The development of the described constant-structure discrete time CSS may be modelled as a discrete time coloring process on a connected graph. From now and till almost the end of the paper I will consider the iterative redefinition process on the vertices only, assuming that the edges are plain connectors, which do not have their own colors and do not participate in the definition of the incident vertices.
Initially, all the vertices have the same ``type'' , so at the first iteration the only difference between any 2 vertices is the number of their neighbors (degree).  For illustrative purposes, we may assign each degree a specific ``color'' . At the second iteration the ``neighbor portrait''\   of each vertex becomes more complicated (e.g., ``5 neighbors''\  at the first iteration becomes ``2 red and 3 green''\  at the second). This iterative coloring process (see Fig.1) produces equitable partitions \cite{godzilla} (perfect colorings, regular partitions and graph divisors), which are well known \cite{unger,weis,arlaz} and were successfully used in e.g. graph isomorphism heuristic Nauty \cite{mckay} (the coloring process in the latter is similar to the Algorithm of this paper).

How will this iterative coloring process behave? Will the size of the palette ever increase or it can decrease and then oscillate? Will the process become self-repeating? If yes, then how fast and what will be the cycle size? The next section gives the answers to these questions.




 



\section{Mathematical model}

\hspace{0.7 cm}There are some little less common mathematical notations used below, which I would rather state explicitly: a) the number of adjacent vertices for each $v\in V$ in simple $G$ is equal to the degree of $v$ and is referred to as $deg(v)$; b) the image $Y$ of the mapping $\phi\colon X\rightarrow Y$ is referred to as $\phi[X]$ and c) an arbitrary (possible multi-valued) mapping between $X$ and $Y$  is referred to as  $X\rightrightarrows Y$.

Let $G=(V,E)$ be a simple graph with vertices $V$ and edges $E$. Each vertex possesses a color, which changes iteratively depending on the colors of the vertex's neighbors. Let $Pal_i\subset\mathbb{N}$ be the set (palette) of colors and $Col_i\colon V\rightarrow Pal_i$ be the coloring function at the $i$-th iteration.
The iterative coloring process may be described as follows:
\ \\

\begin{algorithm*}[h]
 \caption{Infinite iterative vertices coloring process (IICP)}
 $Pal_0:=\{0\}$\;
 $\forall v\in V\ Col_0(v):=0$\;
$\forall v\in V\ Port_0(v):=(k(v))$, where $k(v)=deg(v)$\;
 $i:=1$\;
 \Repeat{False}{
 - let $\chi_i$ be some indexing bijection $\chi_i\colon Port_{i-1}[V]\leftrightarrow\overline{1,|Port_{i-1}[V]|}$,  then the current palette $Pal_i:=\chi_i[Port_{i-1}[V]]$ (consists of the indices of the elements of $Port_{i-1}[V]$); 
 let $K_i:=|Pal_i|$
\;
 - color the graph: $\forall v\in V\ \ Col_i(v):=\chi_i(Port_{i-1}(v))$\;
 - build a new ``neighborhood portrait''\  of each vertex: $$Port_i(v):=\left(k^i_1(v),\ldots,k^i_j(v),\ldots,k^i_{K_i}(v)\right),$$ where $k^i_j(v)$ is the number of adjacent vertices of $v\in V$ that possess the color $j\in Pal_i$\;
- $i:=i+1$\;
 }
\end{algorithm*}
\ \\

\begin{figure}[h]
\centering
\captionsetup{justification=centering}
\includegraphics[width=0.5\textwidth]{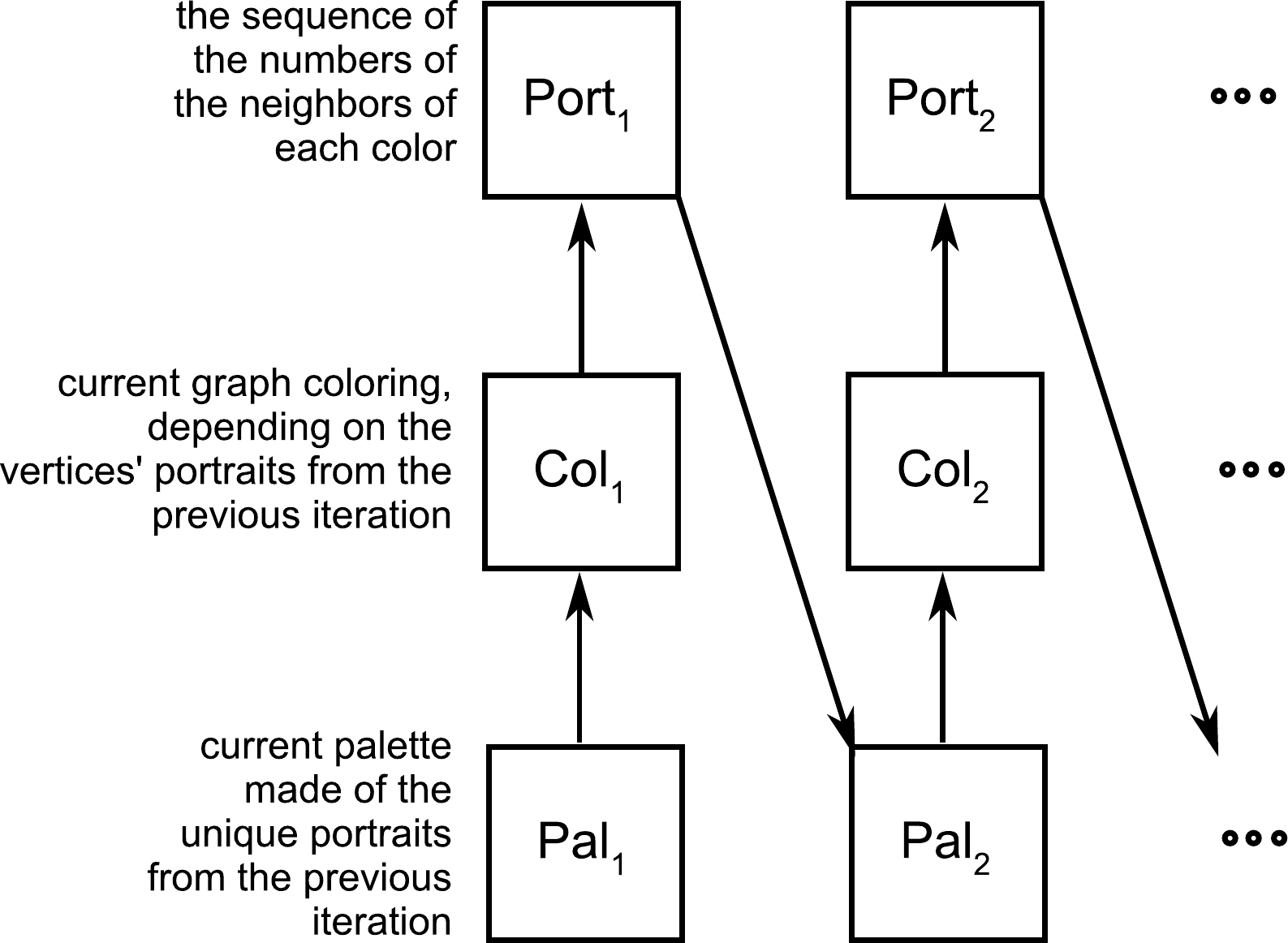}
\caption{Scheme of infinite iterative coloring }
\end{figure}

How will this IICP behave at infinity? Below I prove that, starting from some iteration, all the following colorings will be pairwise isomorphic in the sense of the following definition:

Lemma 1 makes the first step in this direction. It shows that two consequent isomorphic colorings result in the degeneration of the following IICP.

\paragraph{\bf Coloring isomorphism.} Two colorings $Col_{i}$ and $Col_{j}$  are isomorphic ($Col_{i}\sim Col_{j}$), if there exists a bijection
$\phi\colon Pal_{i}\leftrightarrow Pal_j$ such that $\forall v\in V\ \phi(Col_{i}(v))=Col_j(v)$.



\begin{lemma} If $Col_{L-1}\sim Col_{L}$ for some $L$, then all the following colorings will be pairwise isomorphic: $\forall L'\ge L\ \forall L'' \ge L\ Col_{L'}\sim Col_{L''}$. 
\end{lemma}
\begin{proof}
It is enough to prove that $Col_{L}\sim Col_{L+1}$. Let us select an arbitrary $v\in V$ and
take the portrait of $v$ from the previous iteration, which corresponds to 
$Col_{L+1}(v)$:
\begin{equation}
\label{first_xi}
\left(k^L_{1}(v),\ldots,k^L_{K_L}(v)\right)=\chi_{L+1}^{-1}(Col_{L+1}(v)).
\end{equation}

Let us consider the mapping $\psi\colon Port_{L-1}[V]\rightarrow Port_{L}[V]$:
\begin{equation}
\label{varphi}
\psi\left(\left(k^{L-1}_{1}(v),\ldots,k^{L-1}_{K_{L-1}}(v)\right)\right)=\left(k^L_{1}(v),\ldots,k^L_{K_L}(v)\right)
\end{equation}
and prove that it is a bijection.

1) $\psi$ is defined for all elements of $Port_{L-1}$, since it is defined for all $v\in V$.

2) Each portrait $p$ from $Port_{L}$ corresponds to at least one vertex $v'\in V$ at the $L$-th iteration; selecting this $v'$ in \eqref{varphi}, we get at least one preimage for $p$ in $Port_{L-1}$, which means that $\psi$ is surjective.

3) The single-valuedness of $\psi$ is slightly more complicated to prove; suppose there is a portrait from $Port_{L-1}$
\begin{equation}
\label{eq_port1}
p=\left(k^{L-1}_{1}(v_1),\ldots,k^{L-1}_{K_{L-1}}(v_1)\right)=\left(k^{L-1}_{1}(v_2),\ldots,k^{L-1}_{K_{L-1}}(v_2)\right)
\end{equation}
that is mapped by $\psi$ into two portraits from $Port_{L}$
$$p_1=\left(k_1^{L}(v_1),\ldots,k_{K_{L}}^{L}(v_1)\right)\ \text{and}\ \ p_2=\left(k_1^{L}(v_2),\ldots,k_{K_{L}}^{L}(v_2)\right).$$

Since  $Col_{L-1}\sim Col_{L}$, there exists an isomorphism $\varphi\colon Pal_{L-1}\leftrightarrow Pal_L$ such that 
\begin{equation}
\label{k_iso}
\forall v\in V\ \ \forall i\in\overline{1,K_L}\ \ k^{L-1}_{i}(v)=k^{L}_{\varphi(i)}(v).
\end{equation}
Considering this, we can write
$$\forall i\in\overline{1,K_L}\ \ \ k^{L}_{i}(v_1)\stackrel{\eqref{k_iso}}{=}k^{L-1}_{\varphi^{-1}(i)}(v_1)\stackrel{\eqref{eq_port1}}{=}k^{L-1}_{\varphi^{-1}(i)}(v_2)\stackrel{\eqref{k_iso}}{=}k^{L}_{\varphi(i)}(v_2),$$
which means that $p_1$ and $p_2$ are equal elementwise and proves the single-valuedness of $\psi$.


4) The injectiveness of $\psi$ may be proved in the same manner; suppose there are two portraits from $Port_{L-1}$ $$p_1=\left(k_1^{L-1}(v_1),\ldots,k_{K_{L-1}}^{L-1}(v_1)\right)\ \text{and}\ \ p_2=\left(k_1^{L-1}(v_2),\ldots,k_{K_{L-1}}^{L-1}(v_2)\right)$$  that are mapped by $\psi$ into one portrait from $Port_{L}$
\begin{equation}
\label{eq_port2}
\left(k^L_{1}(v_1),\ldots,k^L_{K_L}(v_1)\right)=\left(k^L_{1}(v_2),\ldots,k^L_{K_L}(v_2)\right).
\end{equation}
Returning to \eqref{k_iso}, we have
$$\forall i\in\overline{1,K_L}\ \ \ k^{L-1}_{i}(v_1)\stackrel{\eqref{k_iso}}{=}k^L_{\varphi(i)}(v_1)\stackrel{\eqref{eq_port2}}{=}k^L_{\varphi(i)}(v_2)\stackrel{\eqref{k_iso}}{=}k^{L-1}_{i}(v_2),$$
which proves that $\psi$ is injective and therefore is a bijection.

Applying $\psi^{-1}$ to the right-hand side of \eqref{first_xi}, we get an element of $Port_{L-1}$, which can be mapped further to a color from $Pal_L$ by $\chi_{L}$. 
The final sought-after mapping $\phi\colon Pal_{L}\leftrightarrow Pal_{L+1}$ is the composition of the three above-mentioned mappings,
$$\phi=\chi_{L+1}\circ\psi\circ\chi^{-1}_{L},$$
which is a bijection since all the composing mappings are bijections.
\end{proof}

At this point, we know that two consequent isomorphic colorings make the further coloring process degenerated, but is this situation impossible, probable or inevitable? The answer to this question starts with the analysis of the behavior of the palette size. The number of colors in the palettes evidently belongs to $\overline{1,|V|}$, and intuitively this number either increases or remains constant during the IICP. 

\begin{lemma} The size of the palette never decreases: $\forall i\in\mathbb{N}\ \forall v_1,v_2\in V\ (Col_{i+1}(v_1)= Col_{i+1}(v_2))\Rightarrow (Col_i(v_1)= Col_i(v_2))$.
\end{lemma}

\begin{proof}
Let us use induction on $i$. The base case for $i=0$ is true since all the vertices have the same color at the first iteration: $\forall v\in V\ Col_0(v)=0$. 

Inductive step: assume the statement is true for all $i< L$ and prove it for $i=L$. Let $V_1=\{v_1^1,\ldots,v_1^n\}$ be the 	adjacent vertices of $v_1$ and $V_2=\{v_2^1,\ldots,v_2^m\}$ be the adjacent vertices of $v_2$ in $G$. 

Since $Col_{i+1}(v_1)= Col_{i+1}(v_2)$ and $\chi_{L+1}$ is a bijection,
$$Port_{L}(v_1)=\chi^{-1}_{L+1}(Col_{i+1}(v_1))=\chi^{-1}_{L+1}(Col_{L+1}(v_2))=Port_{L}(v_2),$$
which means that $n=m$ and there exists a bijection $\psi\colon V_1\leftrightarrow V_2$, such that $$\forall j\in\overline{1,n}\ \ Col_L(\psi(v_1^j))=Col_L(v_1^j).$$ By induction, 
$$\forall j\in\overline{1,n}\ \ (Col_L(\psi(v_1^j))=Col_L(v_1^j))\Rightarrow (Col_{L-1}(\psi(v_1^j))=Col_{L-1}(v_1^j)),$$
which means that $Port_{L-1}(v_1)=Port_{L-1}(v_2)$ and thus
$$Col_L(v_1)=\chi_L(Port_{L-1}(v_1))=\chi_L(Port_{L-1}(v_2))=Col_L(v_2).$$
\end{proof}

Now we know that at each iteration the number of colors either increases or remains constant. We also know that the palette size cannot be arbitrary large, it is limited by $|V|$. Putting together both results, we get that an IICP may contain only a finite number of the steps where the palette size grows. In not more than $|V|$ iterations, an IICP will come to the situation, where the palettes at two consequent iterations have the same size. The consequences will be radical.


\begin{lemma} Either the size of the palette increases or the two last colorings are isomorphic: $(K_{i}=K_{i+1})\Rightarrow (Col_{i}\sim Col_{i+1})$.
\end{lemma}
\begin{proof}
I will prove that if $K_{i}=K_{i+1}$, then the mapping  $\phi\colon Pal_{i}\rightrightarrows Pal_{i+1}$, defined as $\phi(Col_{i}(v))=Col_{i+1}(v)\ \forall v\in V$, is a bijection. Let's argue by contradiction: suppose $K_{i}=K_{i+1}$, but $\phi$ is not a bijection. Since $\phi$ is defined for all $v\in V$, it is completely defined on $Pal_i$ and is a surjection onto $Pal_{i+1}$; it is also injective by Lemma 2. The last option for $\phi$ not to be a bijection is multi-valuedness. Suppose that
\begin{equation}
\label{assu}
\exists v',v'' \in V\colon (v'\ne v'' )\ \&\ (Col_{i}(v')=Col_{i}(v'' ))\ \&\ (Col_{i+1}(v')\ne Col_{i+1}(v'' )).
\end{equation}
Let $V_1=\{v\in V\colon Col_{i}(v)=Col_{i}(v')\}$ and $V_2=V\setminus V_1$ (see Fig.3). The set $V_2$ is not empty, since otherwise $$|Col_i[V_1]|+|Col_i[V_2]|=1+0=K_{i}= K_{i+1}\stackrel{\eqref{assu}}{\ge} 2.$$  

There are no colors from $Col_{i+1}[V_1]$ among the colors from $Col_{i+1}[V_2]$ since otherwise (see Fig.3) $\exists v^*_1\in V_1, v^*_2\in V_2\colon Col_{i+1}(v^*_1)=Col_{i+1}(v_2^*)$ but, by construction of $V_1,V_2$,  $ Col_{i}(v^*_1)=Col_{i}(v')\ne Col_{i}(v^*_2)$, which contradicts the injectiveness of $\phi$. Thus, $$Col_{i+1}[V_1]\cap Col_{i+1}[V_2]=\varnothing\ \ \text{and}\ \ K_{i+1}=|Col_{i+1}[V_1]|+|Col_{i+1}[V_2]|.$$

 The number of distinct colors in $Col_{i+1}[V_2]$ is not less than $|Col_{i}[V_2]|=K_{i}-1$ since $\phi$ is injective and does not map distinct colors into the same one.
 
 By the definition of $V_1$ and assumption \eqref{assu}, we have $v',v'' \in V_1$ and  $$|Col_{i+1}[V_1]|\ge |\{Col_{i+1}(v'),Col_{i+1}(v'' )\}|=2,$$ which means that 
 $$K_{i+1}=|Col_{i+1}[V_1]|+|Col_{i+1}[V_2]|\ge 2+(K_{i}-1)=K_{i}+1$$ 
 and contradicts the assumption $K_{i}= K_{i+1}$. 
\end{proof}

\begin{figure}[H]
\centering
\captionsetup{justification=centering}

\includegraphics[width=0.7\textwidth]{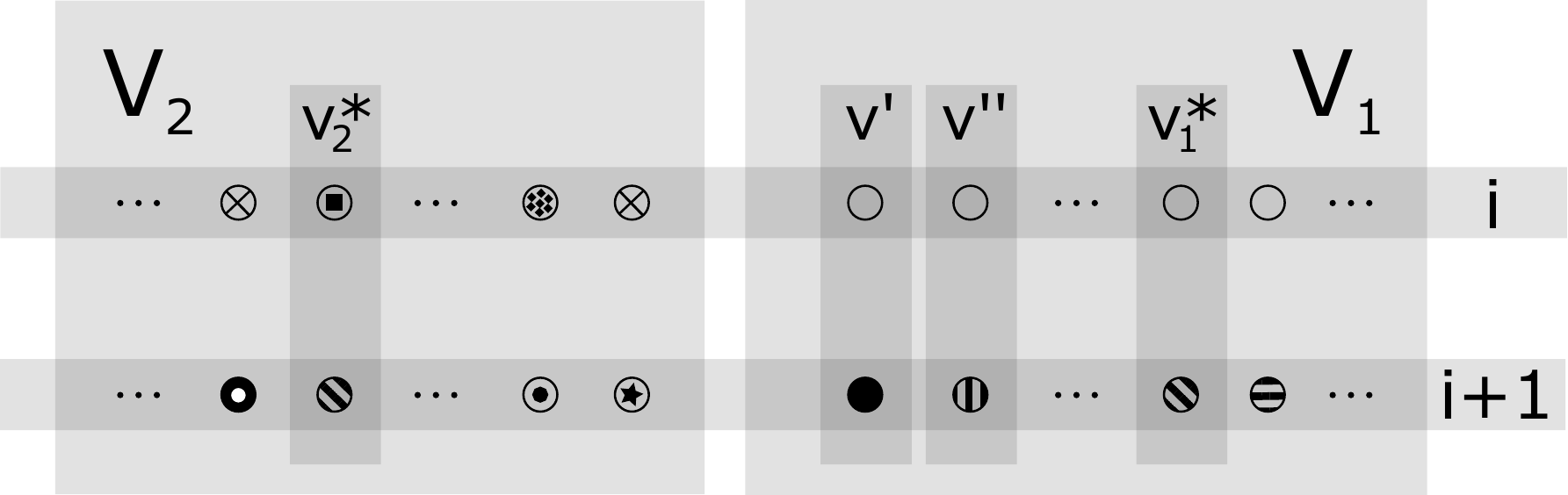}

\caption{Illustration for Lemma 3}
\end{figure}

The results of the previous lemmas may be summed up as a theorem.

\begin{theorem}
	The coloring process described in the Algorithm converges to an unchanging equitable partition in at most $|V|$ iterations.
\end{theorem}	



\paragraph{Edge coloring}Until now, all objects in the CSS were defined through other objects by means of one-type connections. One of the important generalizations of such approach is the introduction of connections of different types. As far as we talk about \textit{closed} semantic systems, the ``type''\  of a connection should be defined within the system. By analogy to the objects defined through their neighbors, the connections could be defined through the objects, which they relate. In the considered graph model of CSS, it means that the edges, just as the vertices, are subjected to an iteration coloring process based on the portraits of neighbors. During such an expanded coloring process, the color of a vertex is defined by the colors of the adjacent vertices and incident edges, and the color of an edge is defined by the colors of its 2 incident vertices.

	
This iterative coloring process for vertices and edges may be converted to the vertices-only coloring case  by the addition of ``virtual''\  vertices corresponding to the edges (Fig. 4) assuming the colors for the vertices and for the edges are taken from the same palette and assigned in accordance with the same laws.


\begin{figure}[H]
	\centering
	\captionsetup{justification=centering}
	
	\includegraphics[width=0.4\textwidth]{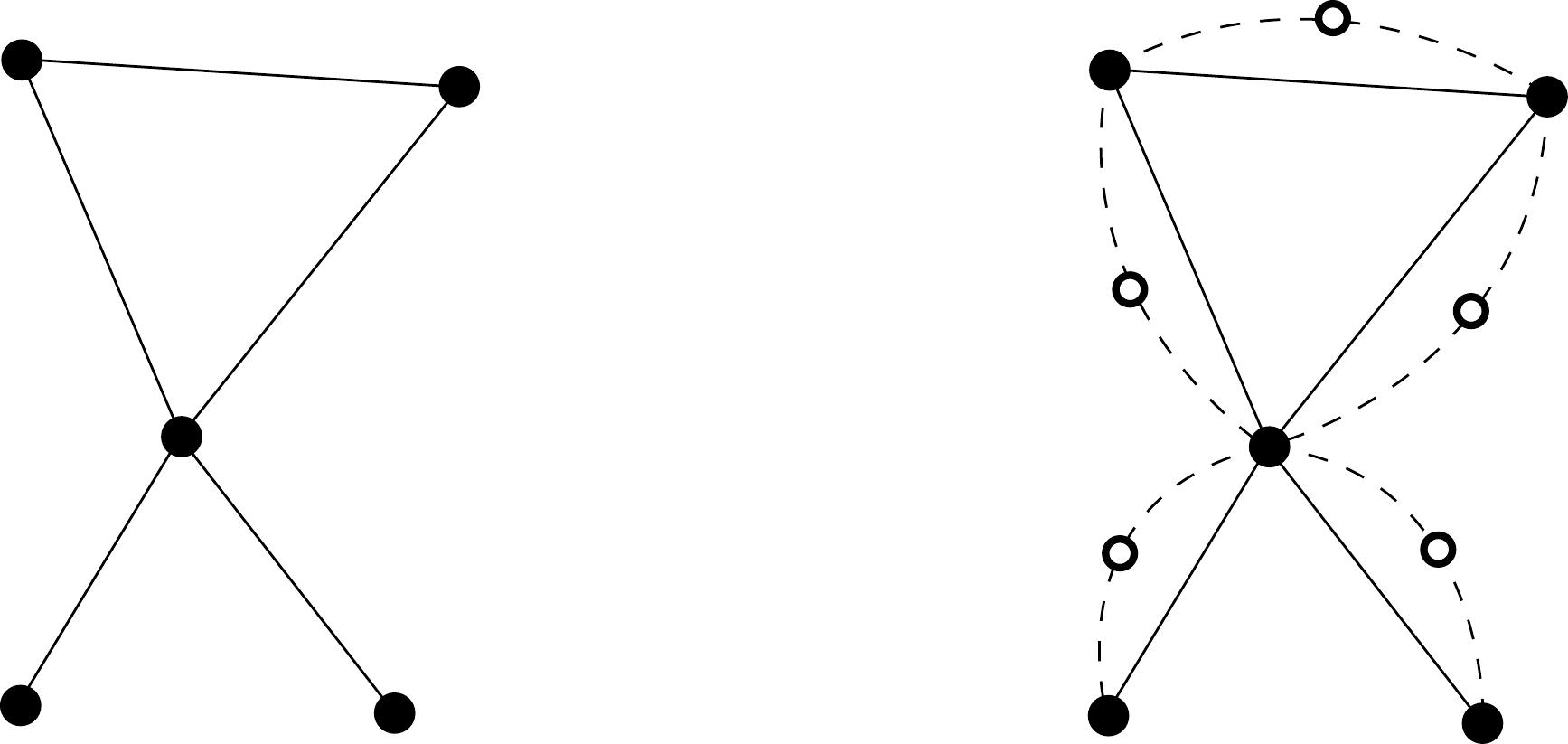}
	
	\caption{Transformation of graph to consider the influence of edges}
\end{figure}

\section{Discussion}

\hspace{0.7 cm}In the present paper I showed that each CSS with a constant structure and discrete time ceases to change rather quickly. The future research could aim at getting the CSS model closer to the real world. The consideration of more complex unsynchronized dynamical structures will change, to all appearances, the further research methods from pure mathematical to statistical and computational.
Below I concern several conceivable directions for the further investigations: 

1. The first step towards real complex systems is to get rid of the discreteness of time. In the majority of real complex systems it seems unnatural to change the states of all elements simultaneously. Such desynchronization will mix the palettes, so the corresponding formal coloring process should be considerably different.

2. The second step is the consideration of the systems that start from an arbitrary state, not only the ``zero state''\  of Algorithm, where $\forall v\in V\ Col_0(v)=0$. It is easy to see that in this case Lemma 2 is not always true (see Fig.5), therefore, Lemma 3 would also need a different proof since the current one relies on Lemma 2. Nevertheless, it seems the proof of convergence can be generalized to the case of arbitrary initial states. 
\begin{figure}[H]
	\centering
	\captionsetup{justification=centering}
	
	\includegraphics[width=0.4\textwidth]{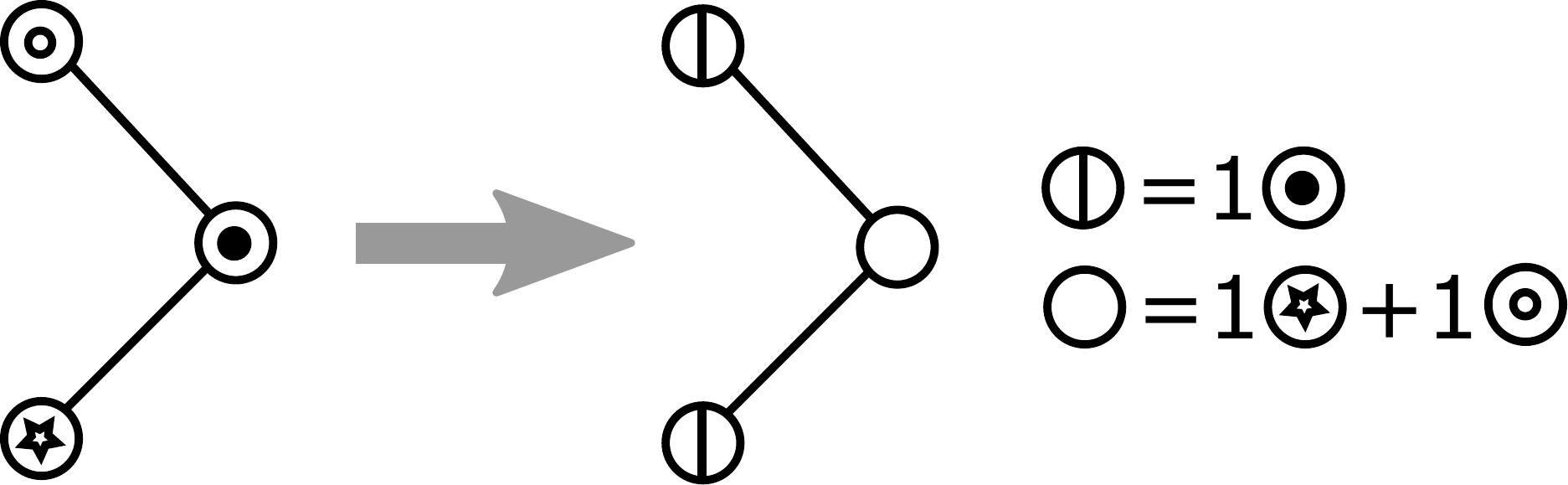}
	
	\caption{Example of the palette's growth starting from a non-zero initial state}
\end{figure}
3. The most interesting research direction is the analysis of behavior of the CSSs with (stochastically) varying sets of objects and connections. Such CSSs could presumably be addressed by means of applied statistics and multi-agent modelling. 
In this respect, the underlying graph model could be replaced with the model of 
horizontal gene transfer \cite{gene} where each agent is ``defined''\  through the influence of the agents encountered during stochastic motion.


\vspace{5mm}{\it The author thanks Dr. D.V.~Khlopin and Dr. A.L.~Gavrilyuk for their comments. 


\begin{thebibliography}{9}

\bibitem{bloom}
Bloom~P.
\newblock {\em How Children Learn the Meanings of Words}.
\newblock A Bradford Book. MIT Press, 2002.

\bibitem{papa}
Papadimitriou C.H.
\newblock {\em Computational Complexity}.
\newblock Pearson, 1993.

\bibitem{godzilla}
Godsil C.D. 
\newblock {\em Algebraic Combinatorics}.
\newblock Chapman and Hall Mathematics Series. Chapman and Hall, 1993.

\bibitem{unger}
Unger S.H.
\newblock {GIT} -- a heuristic program for testing pairs of directed line
  graphs for isomorphism.
\newblock {\em Commun. ACM}, 1964, vol. 7, no. 1, pp. 26--34.

\bibitem{weis}
Weisfeiler B., Lehman A.A.
\newblock [A reduction of a graph to a canonical form and an algebra arising
  during this reduction] (in {R}ussian).
\newblock {\em Nauchno-Technicheskaya Informatsia} [Scientific-technical information], 1968, vol. 2, no. 9, pp. 12--16.

\bibitem{arlaz}
Arlazarov V.L., Zuev I.I., Uskov A.V., Faradzhev I.A.
\newblock An algorithm for the reduction of finite non-oriented graphs to canonical form.
\newblock {\em USSR Computational Mathematics and Mathematical Physics}, 1974, vol. 14, no. 3, pp. 195--201.

\bibitem{mckay}
McKay B.D., Piperno A.
\newblock Practical graph isomorphism, II.
\newblock {\em Journal of Symbolic Computation}, 2014, no. 60, pp. 94--112.

\bibitem{gene}
Syvanen M.
\newblock Horizontal gene transfer: Evidence and possible consequences.
\newblock {\em Annual Review of Genetics}, 1994, vol. 28, no. 1, pp. 237--261.

\end{thebibliography}

\end{document}